\documentclass{article}

\usepackage{microtype}
\usepackage{graphicx}
\usepackage{subfigure}
\usepackage{booktabs} 

\usepackage{hyperref}

\newcommand{\sysname}{dPASP}

\usepackage[accepted]{icml2023}

\usepackage{amsmath}
\usepackage{amssymb}
\usepackage{mathtools}
\usepackage{amsthm}
\usepackage{thmtools}
\usepackage{thm-restate}

\usepackage[capitalize,noabbrev]{cleveref}
\usepackage[frozencache]{minted}
\usepackage{xcolor}
\usepackage{colonequals}
\usepackage{enumitem}
\usepackage{tikz}
\usepackage{svg}
\usepackage{pgfplots}

\pgfplotsset{compat=1.18}

\usepackage{times}
\usepackage{soul}
\usepackage{inconsolata}

\definecolor{palette1}{HTML}{003049}
\definecolor{palette2}{HTML}{a62639}
\definecolor{palette3}{HTML}{f77f00}
\definecolor{palette4}{HTML}{3bb273}
\definecolor{palette5}{HTML}{4d9de0}
\definecolor{palette6}{HTML}{d7cf07}
\definecolor{palette7}{HTML}{97b1a6}
\definecolor{palette8}{HTML}{386641}

\definecolor{mintedframe}{HTML}{5ca4a9}
\setminted{
    fontsize=\footnotesize,
    escapeinside=&&,
    frame=leftline,
    rulecolor=mintedframe,
}
\newminted[pasp]{pasp}{mathescape}
\newmintinline[paspi]{pasp}{mathescape}
\newmintedfile[paspf]{pasp}{mathescape}

\newcommand{\head}[1]{\mathsf{head}(#1)}
\newcommand{\body}[1]{\mathsf{body}(#1)}
\newcommand{\bodyp}[1]{\mathsf{body}^+(#1)}
\newcommand{\bodyn}[1]{\mathsf{body}^-(#1)}
\newcommand{\pr}{\mathbb{P}}
\newcommand{\set}[1]{\mathbf{#1}}

\theoremstyle{plain}

\theoremstyle{definition}

\theoremstyle{remark}

\usepackage[disable,textsize=tiny]{todonotes}

\icmltitlerunning{\sysname: Differentiable Probabilistic Answer Set Programming}

\begin{document}

\twocolumn[
\icmltitle{\sysname: A Comprehensive Differentiable Probabilistic Answer Set Programming Environment For Neurosymbolic Learning and Reasoning}



\icmlsetsymbol{equal}{*}

\begin{icmlauthorlist}
\icmlauthor{Renato Lui Geh}{ime}
\icmlauthor{Jonas Gonçalves}{ime}
\icmlauthor{Igor Cataneo Silveira}{ime}
\icmlauthor{Denis Deratani Mauá}{ime}
\icmlauthor{Fabio Gagliardi Cozman}{poli}
\end{icmlauthorlist}

\icmlaffiliation{ime}{Instituto de Matemática e Estatística, Universidade de São Paulo, São Paulo, Brazil.}
\icmlaffiliation{poli}{Escola Politécnica, Universidade de São Paulo, São Paulo, Brazil}

\icmlcorrespondingauthor{Renato Lui Geh}{renatogeh@gmail.com}
\icmlcorrespondingauthor{Denis Deratani Mauá}{ddm@ime.usp.br}
\icmlcorrespondingauthor{Fabio Gagliardi Cozman}{fgcozman@usp.br}

\icmlkeywords{Machine Learning, Knowledge Representation, Logic Programming, Probabilistic Programming, Probabilistic Logic Programming, Answer Set Programming}

\vskip 0.3in
]



\printAffiliationsAndNotice{}  

\begin{abstract}
We present \sysname,
  a novel declarative probabilistic logic programming framework for differentiable neuro-symbolic reasoning.
  The framework allows for the specification of discrete probabilistic models with neural predicates, logic constraints and interval-valued probabilistic choices, thus supporting models that combine low-level perception (images, texts, etc), common-sense reasoning, and (vague) statistical knowledge.
  To support all such features, we discuss the several semantics for probabilistic logic programs that can express nondeterministic, contradictory, incomplete and/or statistical knowledge.
  We also discuss how gradient-based learning can be performed with neural predicates and probabilistic choices under selected semantics.
  We then describe an implemented package that supports inference and learning in the language, along with several example programs. 
  The package requires minimal user knowledge of deep learning system's  inner workings, while allowing end-to-end training of rather sophisticated models and loss functions. 
\end{abstract}
\section{Introduction}

Answer Set Programming (ASP) \cite{gebser12,lifschitz19} is a powerful declarative paradigm for specifying domain knowledge by means of logic programming.
For example, the following program very intuitively describes the causal relationships between stress, smoking and peer pressure \cite{fierens15}.
\begin{pasp}
smokes(X) :- stress(X).
smokes(X) :- influences(Y,X), smokes(Y).
\end{pasp}
The program can be extended with a database of known facts such as \paspi{influences(anna,bill)} and \paspi{stress(anna)}, and used to conclude that \paspi{smokes(bill)} must be true (according to some selected semantics).

While powerful, ASP cannot cope with uncertainty, that abounds in data-driven situations.
For instance, suppose that we only known (or are only willing to ascertain) that Anna influences Bill with probability 0.8:
\begin{pasp}
0.8::influences(anna,bill).
\end{pasp}
Such probabilistic facts appear quite naturally when they are the output of perception models.
In particular, when differentiable models such as deep neural networks are used, as is the case of image, text and speech recognition.

Once probabilistic facts are allowed, we enter the realm of Probabilistic Answer Set Programming (PASP) \cite{cozman20}; and when the probabilistic facts are linked to the output of neural probabilistic classifiers, we name the resulting framework Neural-Probabilistic Logic Programming (NPLP).
Essentially, NPLP provides a tight coupling of low-level reasoning models (e.g., image recognition) and high-level reasoning (e.g., planning).
By exploiting existing domain knowledge, NPLP allows, among other things, end-to-end gradient-based weakly supervised learning of neural models, thus providing an effective implementation for neurosymbolic reasoning  \cite{yang20,manhaeve21}.

An important and mostly overlooked component of NPLP systems is the interface connecting real-world objects, neural network models and logic programming languages.
Existing implementations of NPLP such as \textsc{DeepProbLog} \cite{manhaeve21} and \textsc{NeurASP} \cite{yang20} require the user to write the ``glue'' between the neural-probabilistic logic language and the deep learning framework themselves, a task which is not always straightforward.
These frameworks are also limited: \textsc{DeepProbLog} forbids negative cycles and nondeterministic knowledge; \textsc{NeurASP} disallows probabilistic facts; and neither of them deals with contradictions or vague uncertain knowledge.

In order to fill those gaps, we introduce {\sysname}, a new NPLP framework that features a large set of ASP constructs, can learn from both probabilistic facts and neural atoms, and supports several semantics, from standard maximum entropy probability measures to richer semantics that can accommodate contradictions and partial specification of probabilistic knowledge.

\section{Background}

We review some background knowledge on logic and probabilistic logic programming.
We assume the reader has some familiarity with the terminology and basics of logic programming \cite{gebser12}, and focus on less common concepts such as L-stable semantics, probability models and the credal L-stable semantics \cite{rocha22}.

\subsection{Logic Programming}

A logic program is a finite set of disjunctive rules of the form:
\begin{pasp}
&$h_1$&; &$\ldots$&; &$h_k$& :- &$b_1$&, &$\ldots$&, &$b_n$&, not &$b_{n+1}$&, &$\ldots$&, not &$b_{n+m}$&.
\end{pasp}
where each $h_i$ and $b_j$ is an atom and
$\mathsf{not}$ denotes default negation.
We say that $\head{r}=\{h_i\}_{i=1}^k$ is the \emph{head} of the rule, while the atoms right of :- are the \emph{body}, denoted $\body{r}$; sets $\bodyn{r}$ and $\bodyp{r}$ denote, resp.,  the sets of negated (i.e., preceded by $\mathsf{not}$) and non-negated atoms in $\body{r}$.
The rule is disjunctive if the head has two or more atoms.
It is called an integrity constraint if the head is empty, and normal if the head is a singleton.
And it is a fact if it is normal and the body is empty.

The \emph{Herbrand base} of a program is the set formed by all ground atoms that can be built using predicate names and constants in the program. 
The \emph{grounding} of a program is the propositional program obtained by grounding each rule, that is, replacing variables with constants from the Herbrand base in every consistent way. 
The semantics of a program with variables is the semantics of its grounding.

Let $t$, $f$ and $u$ denote ground atoms which do \emph{not} occur in a program $L$.
A three-valued interpretation $I$, also called a partial interpretation, is a function from the atoms in the Herbrand base to $\{0,0.5,1\}$ such that $I(t)=1$, $I(f)=0$, $I(u)=0.5$
We define $I^t = \{ A \mid I(A) = 1\}$, $I^f = \{A \mid I(A) = 0\}$ and $I^u = \{ A \mid I(A) = 0.5 \}$ as the sets of true, false and undefined atoms, respectively, according to $I$.
An interpretation is \emph{total} if it assigns every atom (other than $u$) to either true or false.
The interpretation of the positive body of a rule $r$ is the minimum over the value of its atoms,  $I(\bodyp{r}) = \min\{ I(A) \mid A \in \bodyp{r} \}$. The value of the negative body is the minimum of the complements of the atom's values, $I(\bodyn{r}) = \min\{ 1-I(A) \mid A \in \bodyn{r} \}$. The interpretation of the body is $I(\body{p})=\min\{I(\bodyp{r}),I(\bodyn{r})\}$.
Last, the interpretation of the rule's head is the maximum of its atoms, $I(\head{r})=\max\{ I(A) \mid A \in \head{r} \}$.

An interpretation \emph{satisfies} a rule $r$ if and only if $I(\head{r}) \geq I(\body{r})$. 
$I$ is a \emph{model} of a program if it satisfies all of its rules.
We define a partial order $\leq$ (reflexive, antisymmetric and transitive) of interpretations as: 
$I_0 \leq I_1$ if and only if $I_0(A) \leq I_1(A)$ for all $A$.
A model $I$ is minimal if there is no model $I' \leq I$ such that $I \neq I'$. 
If $I$ is total then it is minimal if and only if $I^t$ is $\subseteq$-minimal.
Note that by (3) an interpretation that assigns undefined to all atoms is always a model as long as the program can be rewritten so it does not contain disjunctive rules nor integrity constraints.
Thus, since $\leq$ is a partial order and the Herbrand base is finite, every normal program admits one or more minimal models \cite{przymusinski91}.
This is in contrast to complete (i.e., true/false) semantics, for which a normal program might have none, one or multiple minimal models.

The stability of a model $I$ is connected to the notion of the program's \emph{reduct} w.r.t.\ $I$, written $P/I$, obtained by the modified Gelfond–Lifschitz transformation~\cite{przymusinski91}. 
The transformation operates on each atom $A \in \bodyn{r}$ in the negative body of a rule $r \in L$ and replaces it by the atom $t$, $f$ or $u$ corresponding to its semantics in $I$. Formally, it replaces $A$ with:
(i) $t$ if $A \in I^f$; or
(ii) $f$ if $A \in I^t$; or 
(iii) $u$ if $A \in I^u$. 
We say that $I$ is a \emph{partial stable model} of $L$ if $I$ is a minimal model of $L / I$.
This is the partial stable model semantics (P-stable) for logic programs with disjunctions and default negation \cite{przymusinski91}.
A stable model $I_0$ is \emph{least undefined} if there is no other stable model $I_1$ with $I_0^u\subset I_1^u$. That is, $I$ is least undefined if there is no other stable model that defines (as true or false) more atoms than it.
Then we say that $I$ is a least undefined stable model of $L$, or L-stable model, for short.
This is finally the L-stable model semantics of disjunctive logic programs~\cite{sacca97}.

\subsection{Probabilistic Logic Programming}

Probabilistic logic programming extend logic programs with \emph{annotated disjunctive rules} (ADR) of the form
\begin{pasp}
&$p_1$&::a&$_1$&;&$\ldots$&;&$p_k$&::a&$_k$& :- b&$_1$&,&$\ldots$&,b&$_n$&,not &$b_{n+1}$&,&$\ldots$&,not b&$_{n+m}$&.
\end{pasp}
where $p_1,\dotsc,p_k$ are nonnegative real values whose sum is equal to one.
We also allow ADRs where $\sum_i p_i < 1$.
Semantically, this is a syntactic sugar where the head is extended by atom $f$ (which is never true) with probability $1-\sum_i p_i$.
For example, a probabilistic fact is an ADR with a singleton head and empty body, written as:
\begin{pasp}
&$\theta$&::a.    
\end{pasp}
It is equivalent to the ADR
\begin{pasp}
&$\theta$&::a; &$1-\theta$&::&$f$&.
\end{pasp}

A probabilistic logic program is a finite set of ADRs and (non-probabilistic) rules.
The semantics of probabilistic logic programs extends that of Sato's distribution semantics \cite{sato95}, in which independent random choices induce logic programs.

A \emph{total choice} independently selects one atom of the head of each ADR in a probabilistic logic program.
Note that a total choice can select different atoms for rules with the same head (and different bodies).
Each total choice $\theta$ induces a logic program where each ADR $r$ is transformed into a normal rules
\begin{pasp}
&$h$& :- &$b_1$&, &$\ldots$&, not &$b_m$&.    
\end{pasp}
where $h$ is the atom selected by the total choice for rule $r$, and $b_1,\dotsc,b_m$ are the body. 

An interpretation of a probabilistic logic program is simply an interpretation of the logic program obtained by dropping probabilities (i.e., turning ADRs into dijsunctive rules).
Let $\theta$ denote both a total choice and its induced logic program, and $\Gamma(\theta)$ denote the selected  models of $\theta$, according to some semantics.
A \emph{probability model} is a probability measure $\pr$ over the interpretations $I$ of the program such that (i) $\pr(I) > 0$ if and only if $I$ is in $\Gamma(\theta)$, (ii) $\pr(\Gamma(\theta)) = \prod_{r} p_r$, where the product runs over all ADRs in the probabilistic logic program and $p_r$ is the probability annotating $\theta(r)$ in rule $r$.
The \emph{maximum entropy} probability model, or max-ent model for short, is the probability model that splits $\pr(\Gamma(\theta))$ evenly over each stable model, that is, $\pr(I)=\sum_{\theta: I \in \Gamma(\theta)} (1/|\Gamma(\theta)|)$ for any stable model $I$.

The \emph{credal semantics} assigns to each atom $a$ a pair of lower and upper probabilities such that $\underline{\pr}(a) = \min_{\pr} \pr(a)$ and $\underline{\pr}(a) = \max_{\pr} \pr(a)$, where the optimizations are over the set of probability models (which is closed and convex).
The max-ent semantics assigns to $a$ the probability of the corresponding max-ent probability model.
Note that by definition we have that $\pr_\text{max-ent}(a) \in [\underline{\pr}(a), \overline{\pr}(a)]$.

Note that the previous probabilistic semantics are agnostic with respect to the logic semantics used to produce $\Gamma(\theta)$, requiring only that it is non-empty for each $\theta$.
In the literature, such a set is more commonly selected as the set of stable models \cite{cozman20} or well-founded models \cite{fierens15}.

As we will discuss later,  {\sysname} supports the combination of either the stable, partial, L-stable or \textsc{SMProbLog} semantics, for the logic part, and the credal or max-ent semantics, for the probabilistic part.
The result is thus one of eight different semantics, to be selected by the user and according to the task at hand.
Not every feature is however available for each selected semantics: for instance, learning under credal or L-stable semantics remains an open problem.

\section{The {\sysname} Framework}

We now describe the language as well as inference and learning routines of the {\sysname} framework.

\subsection{Language}

{\sysname} extends \textsc{clingo}'s syntax \cite{gebser17} to neural-probabilistic logic programs by including annotated disjunctive rules and \emph{neural annotated disjunctive rule} (NADR). The latter are ADR whose probabilities are set by a(n externally defined) neural network.

The neural networks that parametrize NADRs are specified using standard deep learning frameworks such as PyTorch.
To facilitate the integration between the probabilistic logic program and the deep learning framework, {\sysname} allows the embedding of Python code within the program via the \paspi{#python} guard.
Code within this guard is executed and all functions declared are available for use within special predicates in the program.
As an example, consider the following \paspi{#python} code prototype for constructing a neural network in PyTorch for MNIST.
\begin{pasp}
#python
def net(): return ... # Neural network.
# (i=1)st or (i=2)nd half of train set as tensor.
def mnist_tr(i): return ... 
# (i=1)st or (i=2)nd half of test set as tensor.
def mnist_te(i): return ...  
#end.
\end{pasp}

The specification of the interface between raw data (that is fed to the neural network in the python code part) and program constants is managed by a special rule of the form
\begin{pasp}
atom(x) &$\sim$& test(@arg1), train(@arg2).
\end{pasp}
In that rule, \paspi{atom} is a user-defined one-place predicate name used to represent 
an object fed into the neural network, \paspi{x} is a constant identifying  a particular object (since the same network can be used several times in the same program) and \paspi{test} and \paspi{train} are reserved predicates, whose arguments are either paths to CSV files or Python functions as defined in \paspi{#python}.
As an example, consider the example of adding two digit images classified by neural networks, as described in \citet{manhaeve21}; here, the network's inputs are images, with each grounding of the NADR representing one of the two digits in the sum, and different images in the train and test sets.
\begin{pasp}
input(0) &$\sim$& test(@mnist_tr(0)), train(@mnist_te(0)).
input(1) &$\sim$& test(@mnist_tr(1)), train(@mnist_te(1)).
\end{pasp}
Note that \paspi{0} and \paspi{1} are arbitrary constants used to identify the two distinct inputs of the same neural network.
These constants are associated with pairs of images taken either from the test or the train set, depending on the task being solved.

A NADR's head must contain a single (possibly non-ground) atom of the form \paspi{l::f(X, {v&$_1$&,&$\ldots$&,v&$_k$&})}, where \paspi{l} is either a \paspi{?}, indicating the neural ADR is learnable, or a \paspi{!}, in which case its parameters are fixed during learning.
Variable \paspi{X} is used to ground the NADR according to its body, while \paspi{v&$_1$&,&$\ldots$&,v&$_k$&} are the possible values the annotated disjunctive rule may take.
Optionally, an interval may be passed as a shorthand; for instance, \paspi{0..9} is equivalent to writing out all digits from 0 to 9.
The network in the NADR is embedded by either passing a function declared in \paspi{#python} or a GitHub repository in PyTorch Hub format.
If the NADR is learnable, then we may optionally pass parameters to the PyTorch optimizer via the \paspi{with} operator.
Finally, the NADR must then be declared with the data predicate as one of the subgoals in its body.
In our running MNIST addition example, the neural predicate \paspi{digit} must cover all digits from 0 to 9.
\begin{pasp}
?::digit(X, {0..9}) as @net :- input(X).
\end{pasp}

Data embedded into special data predicates (e.g. \paspi{input(0)} and \paspi{input(1)} in the previous example) are passed as PyTorch tensors to the neural networks, allowing for efficient parallelization in the CPU or GPU.

Regular ADRs are declared in a similar fashion to NADRs in {\sysname}, except they allow setting an initial probability value for the optimization when the rule is learnable.
For example:
\begin{pasp}
0.5?::h&$_1$&; ?::h&$_2$&; ?::h&$_3$& :- b&$_1$&, b&$_2$&, &$\ldots$&, b&$_4$&.
\end{pasp}
When no initialization value is given, like in the case of \paspi{h&$_2$&} and \paspi{h&$_3$&}, {\sysname} uniformly distributes the remaining mass to the rest of the atoms.
{\sysname} also supports credal facts in order to specify a credal interval for imprecise inference, where
\begin{pasp}
[0.2, 0.7]::f.
\end{pasp}
indicates \paspi{f} may take a probability as low as \paspi{0.2} and as high as \paspi{0.7}.

Purely logic rules may contain arithmetic operations and comparisons over variables of annotated disjunctive rules as long as they are safe and appear in the body.
\begin{pasp}
sum(Z) :- digit(0, X), digit(1, Y), Z=X+Y.
both_even :- digit(0, X), digit(1, Y), X\2=0, Y\2=0.
\end{pasp}
Querying a partial interpretation (possibly conditioned on another interpretation) is done through the special directive \paspi{#query}.
\begin{pasp}
#query digit(0, 4).          % $\mathbb{P}($digit(1, 4)$)$
#query sum(8)|not both_even. % $\mathbb{P}($sum(4)|$\neg$both_even$)$
\end{pasp}
The \paspi{undef} keyword may be used within \paspi{#query} for querying the probability of an atom being undefined under the L-stable or partial semantics.
Each query returns either a precise probability if under the max-entropy semantics, or a pair of probabilities encoding lower and upper bounds if under the credal semantics.

In order to define the semantics of the program, a special directive \paspi{#semantics} may be used.
It must receive at least one of the supported logic or probabilistic semantics, i.e.\ stable model, partial, L-stable, max-entropy or credal semantics.
If none is given, {\sysname} defaults to the credal and stable models semantics.
If only one is given, the missing one is set to its respective default semantics.
\begin{pasp}
#semantics lstable, maxent. % L-stable; max-entropy.
#semantics credal, partial. % Credal; partial.
#semantics stable.          % Stable; credal.
#semantics maxent.          % Max-entropy, stable.
\end{pasp}

The \paspi{#learn} directive specifies the learning procedure to take place in the program, and receives as parameters the training dataset with the observed atoms, as well as the learning parameters such as learning rate, number of epochs and batch size.
\begin{pasp}
#learn @observed_atoms, lr = 0.001, niters = 5.
\end{pasp}

\subsection{Inference}

The most typical inference one draws with neuro-probabilistic logic programming models is to compute the probability of a query atom, possibly conditional on some evidence.
That is, if $\set{q}=\{q_1,\ldots,q_m\}$ and $\set{e}=\{e_1,\dotsc,e_n\}$ are disjoint set of literals, then we are usually interested in computing
\[
\pr(\set{q} | \set{e}) = \frac{\sum_{I \models \set{q},\set{e}} \pr(\set{q},\set{e})}{\sum_{I \models \set{e}} \pr(\set{e})} \, ,
\]
for some or all probability models $\pr$.

{\sysname} provides exact inference by enumerating total choices and using \textsc{clingo}'s solver to enumerate all models for each induced ASP program.
Performing exact inference by that exhaustive approach, be it under the max-ent or credal semantics, limits the scalability of inference to programs with few NADRs and ADRs.

More scalable approximate inference based on knowledge compilation \cite{Totis2021SMProbLogSM}, sampling \cite{pasocs,pasta} and variational methods are planned features for future versions of {\sysname}, which is currently in early stage development.

\subsubsection{Partial, L-stable and \textsc{SMProbLog} semantics}

Internally, {\sysname} only accepts the stable model semantics when performing inference or learning.
To enable support to different semantics, we implement translation procedures to the stable model semantics.

The partial semantics in {\sysname} is implemented via the translation described in \cite{janhunen06}.
In a nutshell, {\sysname} creates an auxiliary atom and rule for each non-probabilistic atom in the program and duplicates logic rules in order to allow undefined values for non-probabilistic atoms.
The L-stable semantics is implemented by checking, at each total choice, if there exists a stable model for the program: in the positive case, {\sysname} performs inference over the stable models of such a program, otherwise it queries from the translated program's partial stable models.

We also implement the \textsc{SMProbLog} semantics, introduced in \citet{Totis2021SMProbLogSM}.
The main difference between the L-stable and \textsc{SMProbLog}'s semantics is how to deal with undefined atoms.
While in L-stable a model may contain undefined, true and false atoms, in \textsc{SMProbLog}, if an atom is set to undefined in a model, then all atoms in this model must also be undefined.

\subsubsection{Maximum entropy semantics}

If the max-ent semantics is selected, it is sufficient to simply add up the (uniform) probabilities of each model that is consistent with the query; this is the same procedure done in \citet{yang20}. More formally, the probability of some observation $O$ under the max-ent semantics is given by
\begin{equation}
\pr(O)=\sum_{\theta\in\Theta}\pr(\theta)\cdot\frac{N(I_\theta\models O)}{N(\theta)},
\end{equation}
where $\Theta$ is the set of all total choices, and $N(I_\theta\models O)$ and $N(\theta)$ return respectively the number of stable models that are consistent with both $\theta$ and $O$, and the number of stable models consistent with $\theta$.

{\sysname} computes $N(I_\theta\models O)$ and $N(\theta)$ by calling \textsc{clingo}'s solver and counting, for each total choice $\theta$, how many models are consistent with observation $O$ and how many models are in total.
The probability $\pr(\theta)$ is easily computable by simply multiplying the probabilities of each probabilistic and neural component (i.e.\ probabilistic and neural facts, rules and annotated disjunctions), as we assume them to be marginally independent from each other.

\subsubsection{Credal semantics}

For the credal semantics, one is interested in the interval of all probabilities $\pr(\set{q}|\set{e})$ obtained by some probability model.
That interval can be described by its lower and upper values, which in {\sysname} are obtained by the exact algorithm described in \citet{cozman20}.
In short, given query $\set{q}=\{q_1,\ldots,q_m\}$ and evidence $\set{e}=\{e_1,\ldots,e_n\}$ literals, we compute the lower $\underline{\pr}(\set{q}|\set{e})$ and upper $\overline{\pr}(\set{q}|\set{e})$ probabilities by iterating over each total choice $\theta$ and counting the models where ($a$) every model satisfies both $\set{q}$ and $\set{e}$, ($b$) some model satisfies both $\set{q}$ and $\set{e}$, ($c$) every model satisfies $\set{e}$ but does not satisfy some value in $\set{q}$, and ($d$) some model satisfies $\set{e}$ but does not satisfy some value in $\set{q}$.
The credal interval is then $[0, 0]$ if $b+c=0$ and $d>0$, $[1,1]$ if $a+d=0$ and $b>0$, and $[a/(a+d),b/(b+c)]$ otherwise.

Credal facts in {\sysname} are only available when the credal semantics is selected.
To perform inference with credal facts, {\sysname} constructs four multilinear polynomials corresponding to $a$, $b$, $c$ and $d$; each term is a total choice $\theta$, each coefficient is the probability of $\theta$, and variables in the polynomial are $x$ if $X=1$ in $\theta$ or $1-x$ otherwise.
The domain of the polynomial is the cartesian product of all pairs of lower and upper probabilities in credal facts.
The functions $a(\set{x})/(a(\set{x})+d(\set{x}))$ and $b(\set{x})/(b(\set{x})+c(\set{x}))$ are then optimized in order to find the two global minimum and maximum respectively, with the first amounting to the lower and the second the upper probabilities of the queries.

\subsection{Parameter learning}

{\sysname} currently implements three parameter learning rules for the max-ent stable model semantics: (i) a fixed-point learning procedure for non-neural programs, (ii) a Lagrange multiplier derivation for gradient ascent, and (iii) an implementation of \textsc{NeurASP}'s learning procedure \cite{yang20}.
How to learn the parameters of programs in partial or least undefined stable model semantics either under the max-ent or credal semantics is an open problem.

We now describe the first two learning rules, which as far as we are aware, are novel in the literature.
Both rules provide rules for maximizing the log-likelihood $\mathcal{L}(\set{O}) = \sum_{O\in\set{O}}\log\pr(O)$ of a set of observations $\set{O}$ with respect to the parameters $\set{p}$ of the program, which are the probabilities that annotate ADRs.

\subsubsection{Fixed-point parameter learning}

We start with the fixed-point learning procedure, which can be used when we have only ADRs (but no NADR). 
\begin{restatable}{prop}{fixedpoint}\label{thm:fixedpoint}
    Let $\mathbb{P}(X=x)$ be the probability of a specific probabilistic component $X$ we wish to learn from the set of observations $\set{O}$.
    If the iterated application of the rule
    \begin{equation}
        \pr(X=x)=\frac{1}{|\set{O}|}\cdot\sum_{O\in\set{O}}\frac{\pr(X=x,O)}{\pr(O)}.
    \end{equation}
    converges, then it does so to a critical point of the log-likelihood function.
\end{restatable}
The marginal $\pr(X=x,O)$ can be computed by counting the models consistent with both the observation $O$ and $X=x$ and weighting over the probability of the total choices.

\subsubsection{Lagrangian parameter learning}\label{sec:lagrange}

We now derive an update rule that applies also in the presence of NADRs, and is an alternative to \textsc{NeurASP}'s  learning rule \cite{yang20}. 
To better understand the need of our  alternative parameter learning method, we must first understand the shortcomings of the \textsc{NeurASP}'s  learning rule.
The rule updates parameters by $\set{p} \gets \set{p} - \eta \nabla_{\rho}\mathcal{L}(O)$, where $\eta$ is a learning rate and the gradient components are:
\begin{align} \label{na-rule}
\begin{split}
    &\frac{\partial}{\partial p_x}\mathcal{L}(O)= \\
    & \quad \frac{1}{\pr(O)}\sum_{\theta_x}\frac{\pr(\theta_x)}{\pr(X=x)}\cdot\frac{N(I_{\theta_x}\models O)}{N(\theta_x)} \\
    & \qquad - \sum_{\substack{\overline{x},\,\overline{x}\neq x}}\frac{1}{\pr(O)}\sum_{\theta_{\overline{x}}}\frac{\pr(\theta_{\overline{x}})}{\pr(X=\overline{x})}\cdot\frac{N(I_{\theta_{\overline{x}}}\models O)}{N(\theta_{\overline{x}})}.
\end{split}
\end{align}
The intuition is that, interpretations that are consistent with $O$ increase the value of the derivative, while interpretations that are not decrease it.

Note, however, that the sum of the updates over all the parameters $p_x$ of a probabilistic component $X$ is only zero when $X$ is binary, which means that rule \eqref{na-rule} can produce estimates that are outside the feasible set of valid parameters (i.e., they are not probability distributions).
This issue can be mitigated by either projecting the parameters back to the feasible set or ensuring that parameter updates lie within the feasible set, for instance by using a softmax layer.
To avoid this issue, we instead constrain parameters to remain within the feasible set by employing Lagrange multipliers.

\begin{restatable}{prop}{lagrangian}\label{thm:lagrangian}
    The constrained derivative of the log-likelihood function with respect to the probability $\pr(X=x)=p_x$ of a probabilistic component $X$ is 
    \begin{align} \label{lagrangian}
    \begin{split}
        &\frac{\partial}{\partial p_x}\mathcal{L}(O)=\\
        & \, \left(1-\frac{1}{m}\right)\frac{1}{\pr(O)}\sum_{\theta_x}\frac{\pr(\theta_x)}{\pr(X=x)}\cdot\frac{N(I_{\theta_x}\models O)}{N(\theta_x)}\\
        & \; -\frac{1}{m}\sum_{\substack{\overline{x},\,\overline{x}\neq x}}\frac{1}{\pr(O)}\sum_{\theta_{\overline{x}}}\frac{\pr(\theta_{\overline{x}})}{\pr(X=\overline{x})}\cdot\frac{N(I_{\theta_{\overline{x}}}\models O)}{N(\theta_{\overline{x}})}.       
    \end{split}
    \end{align}
    where $m$ is the number of possible values $X$ can take.
\end{restatable}

Interestingly, \eqref{lagrangian} yields a similar expression to \eqref{na-rule}, with the only distinction being the factors $1-\frac{1}{m}$ and $\frac{1}{m}$.
Thus, when $m=2$, the Lagrangian rule is equivalent to halving the learning rate of \textsc{NeurASP}'s rule.
For $m > 2$, rule \eqref{lagrangian} assigns more weight to the probability of interpretations consistent with the observation, and less weight to its complement.
Note that this is more sensible, since the latter sums over more terms than the former.

The extension of \eqref{lagrangian} to the neural case is trivial; by applying the chain rule on the derivative of the log-likelihood with respect to the output $p_x$ of the neural component $X$, we easily find that the resulting gradient is \eqref{lagrangian} multiplied by the derivative of the neural network with respect to network weights $\set{w}$
\begin{equation}
    \frac{\partial}{\partial\set{w}}\mathcal{L}(O)=\frac{\partial\mathcal{L}(O)}{\partial p_x}\frac{\partial p_x(\set{x})}{\partial\set{w}},
\end{equation}
where $\frac{\partial p_x(\set{x})}{\partial\set{w}}$ is the standard backward pass in a neural network.

\section{Experiments}

In this section, we present two preliminary experiments showcasing the {\sysname} system.
The first experiment compares the performance of our system against two competitors on the task of parameter learning in image classification.
The second showcases a possible use of the credal semantics in cautious ensemble classification.

\subsection{MNIST Addition}

We compare the performance of {\sysname} to \textsc{NeurASP}, \textsc{DeepProbLog} and a purely data-driven convolutional neural network (CNN) on the task of learning addition of MNIST image digits, a common distant supervision benchmark for NPLP  \cite{manhaeve21}.
This is a preliminary experiment, as {\sysname} is still in early development.
Given two unlabelled images (e.g.\tikz[baseline=-0.5ex]{\node{\includegraphics[width=1em]{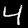}};}and\tikz[baseline=-0.5ex]{\node{\includegraphics[width=1em]{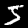}};}) of digits, and the corresponding atom (e.g. \paspi{sum(9)}) as a distant label, the program must learn to identify the sum of digits.

The {\sysname} program to perform this task is quite simple and short if we do not account for the Python code needed for processing the MNIST data.
The program in its entirety can be found in \cref{sec:mnist}.
\begin{pasp}
#python
def net(): ...       # neural network
def mnist_tr(i): ... # train images for i-th digit
def mnist_te(i): ... # test images for i-th digit
def labels(): ...    # sum(x) labels
#end.

input(0) &$\sim$& test(@mnist_te(0)), train(@mnist_tr(0)).
input(1) &$\sim$& test(@mnist_te(1)), train(@mnist_tr(1)).

?::digit(X, {0..9}) as @net :- input(X).
sum(Z) :- digit(0, X), digit(1, Y), Z = X+Y.

#semantics maxent.
#learn @labels, lr = 0.001, niters = 5, batch = 1000.
\end{pasp}

We briefly highlight the fact that the prior know-how needed to write a program in \textsc{NeurASP} or \textsc{DeepProbLog} can be a significant barrier to the widespread use of NPLPs.
Not only is the user required to have a good grasp of Python, but they must also have a significant understanding of the deep learning system used by the NPLP.
For instance, one might compare the equivalent programs in \textsc{NeurASP}\footnote{\url{https://github.com/azreasoners/NeurASP/blob/master/examples/mnistAdd/mnistAdd.py}} and \textsc{DeepProbLog}\footnote{\url{https://github.com/ML-KULeuven/deepproblog/blob/master/src/deepproblog/examples/MNIST/addition.py}} to ours in order to understand the steep learning curve of current NPLPs.

Going back to our preliminary experiment, we follow a similar methodology used in \textsc{NeurASP} \cite{yang20} and \textsc{DeepProbLog} \cite{manhaeve21} for the MNIST digit addition task.
Just like the aforementioned works, we split the original MNIST dataset in half, taking the first (resp.\ second) half as the images of the first (resp.\ second) digit; the labels observed by the program are the atoms corresponding to the sum of the labels of the two halves.
We use the same learning parameters for {\sysname}, \textsc{NeurASP}, \textsc{DeepProbLog} and CNN: a learning rate of $0.001$, batch size of $1000$, and the \textsc{Adam} optimizer for the neural components/networks \cite{kingma17}.

\begin{figure*}
    \resizebox{\textwidth}{!}{
    \begin{tikzpicture}
    \begin{axis}[
        xlabel={Epochs},
        ylabel={Accuracy (\%)},
        xtick={0,...,5},
        minor x tick num=1,
        minor y tick num=1,
        grid=both,
        no markers,
        every axis plot/.append style={ultra thick},
        name=boundary,
    ]
        \addplot[draw=palette3] table {cnn_sum.txt};\label{pgfplots:c4}
        \addplot[draw=palette2] table {neurasp_norm.txt};\label{pgfplots:c2}
        \addplot[draw=palette4] table {deepproblog_norm.txt};\label{pgfplots:c3}
        \addplot[draw=palette1] table {dpasp_corr_norm.txt};\label{pgfplots:c1}
        \addplot[draw=palette5] table {dpasp_500_corr_norm.txt};\label{pgfplots:c5}
    \end{axis}
    \end{tikzpicture}
    \begin{tikzpicture}
    \begin{axis}[
        xlabel={Epochs},
        xtick={0,...,5},
        minor x tick num=1,
        minor y tick num=1,
        grid=both,
        no markers,
        every axis plot/.append style={ultra thick},
        name=boundary,
    ]
        \addplot[draw=palette6] table {cnn_digit.txt};\label{pgfplots:c6}
        \addplot[draw=palette2] table {neurasp_nn_norm.txt};
        \addplot[draw=palette1] table {dpasp_nn_corr_norm.txt};
        \addplot[draw=palette4] table {deepproblog_nn_norm.txt};
        \addplot[draw=palette5] table {dpasp_500_nn_corr_norm.txt};
    \end{axis}
    \node[draw,thick,fill=white,inner sep=0pt,xshift=1cm,yshift=1.1cm] at (boundary.south) {
        \scriptsize
        \begin{tabular}{clc}
            & & \textsc{Time}\\
            \ref{pgfplots:c1} & {\sysname} 1000 batch & 14s\\
            \ref{pgfplots:c5} & {\sysname} 500 batch & 20s\\
            \ref{pgfplots:c2} & \textsc{NeurASP} & 9m 17s\\
            \ref{pgfplots:c3} & \textsc{DeepProbLog} & 17m 22s\\
            \ref{pgfplots:c4} & \textsc{CNN sum} & 29s\\
            \ref{pgfplots:c6} & \textsc{CNN digit} & 30s\\
        \end{tabular}
    };
    \end{tikzpicture}
    }
    \caption{Sum and digit classification accuracy and training time for {\sysname}, \textsc{NeurASP}, \textsc{DeepProbLog} and CNN.
    On the left, accuracy per iteration of classifying sums; on the right, accuracy of learned networks on classifying digits.}
    \label{fig:mnist}
\end{figure*}
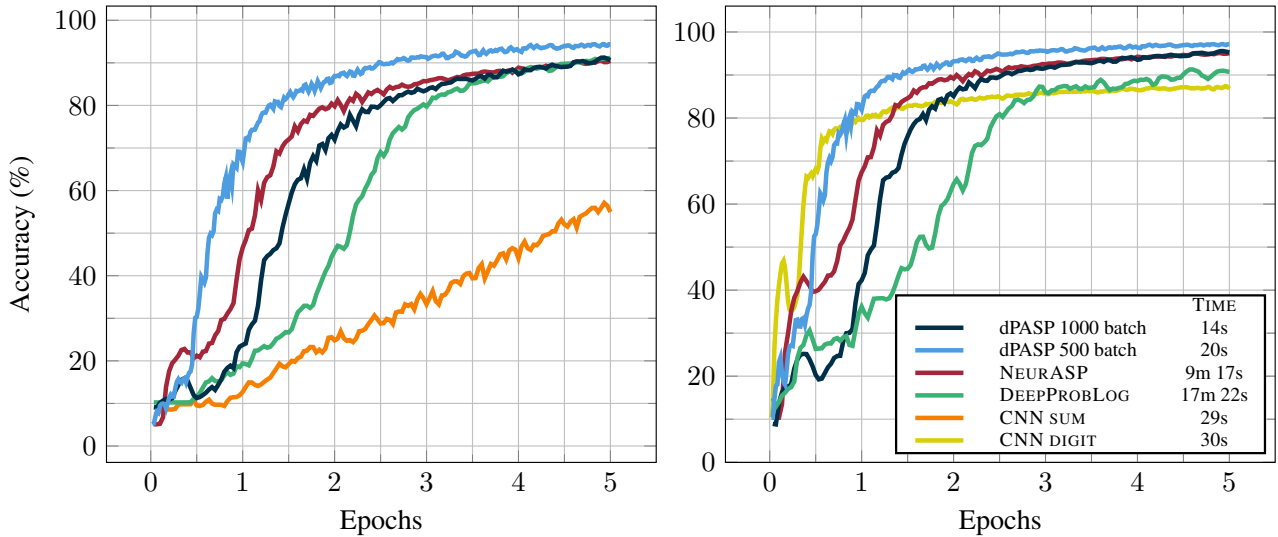

\Cref{fig:mnist} shows a comparison of both the performance in terms of classification accuracy, as well as training time.
The plot on the left compares the accuracy of programs in classifying the correct sum of digits, while the plot on the right shows the digit classification accuracy of the embedded neural networks in the programs while they learn to classify the sum of digits. 
Curve \textsc{CNN sum} corresponds to the performance of evaluating a CNN whose input is a single image consisting of concatenating the two digits and whose output are the probabilities of the 19 possible two-digit sum values; \textsc{CNN digit} is the accuracy of a single digit classification network under the same parameter conditions as {\sysname}, \textsc{NeurASP} and \textsc{DeepProbLog}.

Interestingly, both purely data-driven \textsc{CNN} approaches performed poorly compared to NPLPs.
In particular, \textsc{CNN sum} struggled to even break 50\% accuracy, while \textsc{CNN digit} quickly converged to the 80\% mark, below that of NPLPs.
We again stress the fact that these results were obtained by subjecting all systems to the same learning parameters.

Comparing {\sysname} against \textsc{DeepProbLog} and \textsc{NeurASP}, we find that \textsc{NeurASP} achieves better accuracy faster, although {\sysname} eventually catches up.
This difference between {\sysname} and \textsc{NeurASP} might be explained by the correction factor discussed in \Cref{sec:lagrange}; the factors involved in the Lagrangian optimization slow down learning, as the gradient is diminished due to the correction.

Of note is the significant difference in training time for the four methods, with a surprising gap between {\sysname} and CNN.
We conjecture that the main factor that explains this discrepancy is implementation overhead: {\sysname} is mostly written in C, with only the grammar parsing part of the language implemented in Python; in contrast, even though most of the computation in a pure Python script using PyTorch or any other deep learning framework is also done in C, a lot of the boilerplate code runs in Python.
Thus, the careful implementation of highly optimized C code might make up for the use of logical inference routines in small problems like MNIST addition.

Given that {\sysname} is quite faster compared to the other approaches, we argue that this speed-up is a key advantage of {\sysname} (when the number of probabilistic and neural components is low) and should be exploited.
With this in mind, we lower the batch size of {\sysname} in order to show how, by taking advantage of the careful optimized implementation in {\sysname}, we might achieve better performance by slightly increasing training time.
The {\sysname} batch 500 curve in \Cref{fig:mnist} shows the impact of performance in terms of accuracy when halving the batch size, which causes only a slight increase (6 seconds) in training time.

\subsection{Ensemble Classification}

The following program considers pooling probabilistic predictions of different forecasters without information about accuracy of each forecaster.
The neural predicates \paspi{f} and \paspi{g} are pre-trained MNIST digit classifiers.
Our goal is to perform a digit classification from a two-classifier ensemble and compare a precise strategy that gives equal weight to each of the two classifiers, against a credal strategy that employs a more cautious approach.

\begin{pasp}
data(x) &$\sim$& test(@mnist_test), train(@mnist_train).
% we have two classifiers/forecasters
class(f). class(g).
!::f(X,{0,&$\ldots$&,9}) as @net1 :- data(X).
!::g(X,{0,&$\ldots$&,9}) as @net2 :- data(X).
pred(f,Y) :- f(x,Y).  pred(g,Y) :- g(x,Y).
% at least one of them is correct 
hit(f); hit(g).
% their correctness is consistent
hit(N2)  :- pred(N1,Y), pred(N2,Y), N1 != N2, hit(N1).
miss(N2) :- pred(N1,Y), pred(N2,Y), N1 != N2, miss(N1).
miss(N1); miss(N2) :- pred(N1,Y), pred(N2,Z), Y != Z.
miss(N) :- not hit(N), class(N). 
% classification agrees with correct prediction
digit(Y) :- pred(N,Y), hit(N).
% We must pick one of the below semantics.
#semantics credal. % for credal semantics.
#semantics maxent. % for max-ent semantics.
% Query for the digit probabilities.
#query digit(0). ... #query digit(9).
\end{pasp}

The program encodes the assumption that some of the predictions made by the  forecaster must be correct (hit), and rules outs inconsistent assessments (e.g., two different predictions being considered both correct).

Using our framework and the previous logic program, we can perform either precise or credal inference.
The precise inference amounts to simply computing the probability of each class under the maximum entropy semantics and taking the most likely one.
For the credal semantics, we employ a max-max decision strategy to classify digits; for each digit's lower and an upper probabilities, we select the upper (max) as the probability of being that digit.
Once all upper probabilities have been selected, we then select the class with the highest (max) probability.

\begin{table}
    \centering
    \begin{tabular}{|r|cc|c|} \hline
          & Max-Ent  & Credal          & \# examples \\ \hline
      0   & \textbf{98.36}  & 97.95 & 980 \\
      1   & 98.14  & \textbf{98.23} & 1135\\
      2   & 92.24 & \textbf{92.92} & 1032\\
      3   & \textbf{93.56} & 92.57 & 1010\\
      4   & \textbf{95.92} & 95.72 & 982\\
      5   & 92.48 & \textbf{93.83} & 892\\
      6   & 96.13 & \textbf{96.55} & 958\\
      7   & 94.06 & \textbf{94.35} & 1028\\
      8   & 89.42 & \textbf{89.93} & 974\\
      9   & 46.77 & \textbf{47.77} & 1009\\\hline
      Total & 89.73 & \textbf{89.99} & 10000 \\ \hline
    \end{tabular}
    \caption{Accuracy when classifying each MNIST digit under the max-ent or credal semantics. Best accuracy in bold.}
    \label{tab:precise-maxmax-comparison}
\end{table}

The accuracy for detecting each digit using the two strategies is shown in Table \ref{tab:precise-maxmax-comparison}.
We can see that a more cautious ensemble strategy based on the credal semantics can achieve slightly better performance in terms of accuracy compared to a simple uniform weighting of components.

\section{Conclusion}

We have presented {\sysname}, a new and flexible framework for neurosymbolic reasoning based on probabilistic logic programming.
The framework extends the answer set programming declarative language with probabilistic and neural facts that allows the specification of uncertain knowledge and the tight integration of deep perception models (image classifiers, named entity recognizers, etc) with logic reasoning and constraint solving.

Unlike other similar systems such as \textsc{NeurASP} and \textsc{DeepProbLog}, the framework implementation provides several different semantics for the logic and probabilistic parts, and a more friendly interface between machine learning components (e.g., PyTorch models) and the logic specification.

The system is also relatively fast for learning certain class of models, as illustrated by preliminary empirical results.
This is due to a careful software implementation, which we release as free and open source at \url{https://kamel.ime.usp.br/dpasp}.

There is still much to achieve to make the system more broadly applicable.
In particular, more efficient learning and inference routines need to be devised to scale to larger domains.

\appendix

\section{Proofs}

\fixedpoint*
\begin{proof}
    The log-likelihood function of the program is given by
    \begin{equation}\label{thm:fixedpoint:ll}
        \mathcal{L}(\set{O})=\sum_{O\in\set{O}}\log\sum_{\theta\in\Theta}\pr(\theta)\cdot\frac{N(I_\theta\models O)}{N(\theta)}.
    \end{equation}
    We wish to constrain the probabilities of a probabilistic component $X$ to $\sum_{x\in\mathcal{X}}\pr(X=x)=1$ and $\pr(X=x)>0$, where $\mathcal{X}$ is the set of all possible values $X$ can take.
    To do this, instead of directly computing the derivative of $\mathcal{L}$ with respect to $\pr(X=x)$, we instead compute $\frac{\partial}{\partial w_x}\mathcal{L}$, where $w_x\in\mathbb{R}$ is unconstrained and define
    \begin{equation}
        \pr(X=x)=\frac{e^{w_x}}{\sum\limits_{x'\in\mathcal{X}}e^{w_{x'}}},
    \end{equation}
    i.e.\ we optimize with respect to a softmax function instead of directly working with the probabilities.
    With this in mind, the derivative of the log-likelihood function is given by
    \begin{align}
    \begin{split}
        \frac{\partial}{\partial w_x}&\mathcal{L}(\set{O})=\\
        &\sum_{O\in\set{O}}\frac{1}{\pr(O)}\sum_{\theta\in\Theta}\frac{e^{w_x}}{\sum\limits_{x'\in\mathcal{X}}e^{w_{x'}}}\cdot\pr(\theta_{-x})\cdot\frac{N(I_\theta\models O)}{N(\theta)},
    \end{split}
    \end{align}
    where $\pr(\theta_{-x})$ is the probability of the total choice $\theta$ excluding the probability of $X$, or more formally
    \begin{equation}
        \pr(\theta_{-x})\coloneqq\prod_{\substack{(Y=y)\in\theta\\Y\neq X}}\pr(Y=y)=\frac{\pr(\theta)}{\pr(X=x)}.
    \end{equation}
    We may then split \eqref{thm:fixedpoint:ll} into two terms: one where the total choices $\Theta_x$ agree with the weight $w_x$ to be derived, and another where total choices $\Theta_{\overline{x}}$ choose other values $X=\overline{x}$, $\overline{x}\neq x$ for $X$
    \begin{align}\label{thm:fixedpoint:dx}
    \begin{split}
        &\frac{\partial}{\partial w_x}\mathcal{L}(\set{O})=\sum_{O\in\set{O}}\frac{1}{\pr(O)}\cdot\\
        &\overbrace{\left(\frac{\partial}{\partial w_x}\sum_{\theta_x\in\Theta_x}\frac{e^{w_x}}{\sum\limits_{x'\in\mathcal{X}}e^{w_{x'}}}\cdot\pr(\theta_{-x})\cdot\frac{N(I_{\theta_x}\models O)}{N(\theta_x)}\right.}^{\eqref{thm:fixedpoint:first-dx}}\\
        &+\underbrace{\frac{\partial}{\partial w_x}\left.\sum_{\theta_{\overline{x}}\in\Theta_{\overline{x}}}\frac{e^{w_{\overline{x}}}}{\sum\limits_{x'\in\mathcal{X}}e^{w_{x'}}}\cdot\pr(\theta_{-\overline{x}})\cdot\frac{N(I_{\theta_{\overline{x}}}\models O)}{N(\theta_{\overline{x}})}\right)}_{\eqref{thm:fixedpoint:second-dx}}.
    \end{split}
    \end{align}
    For the sake of clarity, let us call $N_{\theta}^O=\frac{N(I_\theta\models O)}{N(\theta)}$.
    We may simplify the first term in \eqref{thm:fixedpoint:dx} to
    \begin{align}
    \begin{split}\label{thm:fixedpoint:first-dx}
        &\sum_{\theta_x\in\Theta_x}\frac{e^{w_x}\cdot\sum\limits_{x'\in\mathcal{X}}e^{w_{x'}} - e^{w_x}\cdot e^{w_x}}{\sum\limits_{x'\in\mathcal{X}}e^{w_{x'}}\cdot\sum\limits_{x'\in\mathcal{X}}e^{w_{x'}}}\cdot\pr(\theta_{-x})\cdot N_{\theta_x}^O=\\
        &\sum_{\theta_x\in\Theta_x}\left(\frac{e^{w_x}}{\sum\limits_{x'\in\mathcal{X}}e^{w_{x'}}}-\left(\frac{e^{w_{x'}}}{\sum\limits_{x'\in\mathcal{X}}e^{w_{x'}}}\right)^2\right)\cdot\pr(\theta_{-x})\cdot N_{\theta_x}^O=\\
        &\sum_{\theta_x\in\Theta_x}\pr(\theta_x)\cdot N_{\theta_x}^O-\sum_{\theta_x\in\Theta_x}\frac{e^{w_{x'}}}{\sum\limits_{x'\in\mathcal{X}}e^{w_{x'}}}\cdot\pr(\theta_x)\cdot N_{\theta_x}^O=\\
        &\pr(X=x,O)-\pr(X=x)\cdot\pr(X=x,O)
    \end{split}
    \end{align}
    and the second term to
    \begin{align}\label{thm:fixedpoint:second-dx}
    \begin{split}
        &-\sum_{\theta_{\overline{x}}\in\Theta_{\overline{x}}}\frac{e^{w_{\overline{x}}}}{\sum\limits_{x'\in\mathcal{X}}e^{w_{x'}}}\cdot\frac{e^{w_x}}{\sum\limits_{x'\in\mathcal{X}}e^{w_{x'}}}\cdot\pr(\theta_{-\overline{x}})\cdot N_{\theta_{\overline{x}}}^O=\\
        &-\sum_{\theta_{\overline{x}}\in\Theta_{\overline{x}}}\pr(X=x)\cdot\pr(\theta_{\overline{x}})\cdot N_{\theta_{\overline{x}}}^O=\\
        &\pr(X=x)\cdot\pr(X=\overline{x},O).
    \end{split}
    \end{align}
    Putting \eqref{thm:fixedpoint:first-dx} and \eqref{thm:fixedpoint:second-dx} together, we get
    \begin{align}\label{thm:fixedpoint:together}
    \begin{split}
        &\frac{\partial}{\partial w_x}\mathcal{L}(\set{O})=\sum_{O\in\set{O}}\frac{1}{\pr(O)}\bigl[\pr(X=x,O)-\\
        &\pr(X=x)\cdot\pr(X=x,O)-\pr(X=x)\cdot\pr(X=\overline{x},O)\bigr]=\\
        &\quad\sum_{O\in\set{O}}\frac{1}{\pr(O)}\bigl[\pr(X=x,O)-\pr(X=x)\cdot\pr(O)\bigr].
    \end{split}
    \end{align}
    By setting the derivative of the objective function to zero to find the critical point, we finally find that
    \begin{align}\label{thm:fixedpoint:final}
    \begin{split}
        &\sum_{O\in\set{O}}\frac{1}{\pr(O)}\bigl[\pr(X=x,O)-\pr(X=x)\cdot\pr(O)\bigr]=\\
        &\sum_{O\in\set{O}}\left(\frac{\pr(X=x,O)}{\pr(O)}-\pr(X=x)\right)=\\
        &\sum_{O\in\set{O}}\frac{\pr(X=x,O)}{\pr(O)}-|\set{O}|\cdot\pr(X=x)=0\\
        &\implies\pr(X=x)=\frac{1}{|\set{O}|}\sum_{O\in\set{O}}\frac{\pr(X=x,O)}{\pr(O)}.
    \end{split}
    \end{align}
\end{proof}

\lagrangian*
\begin{proof}
    Recall that the log-likelihood function is given by
    \begin{equation}
        \mathcal{L}(\set{O})=\sum_{O\in\set{O}}\log\sum_{\theta\in\Theta}\pr(\theta)\cdot\frac{N(I_\theta\models O)}{N(\theta)},
    \end{equation}
    and that the derivative of the log-likelihood function with respect to a probabilistic component $p_x=\pr(X=x)$ is given by the expression
    \begin{equation}\label{thm:lagrangian:lldx}
        \frac{\partial}{\partial p_x}\mathcal{L}(\set{O})=\sum_{O\in\set{O}}\frac{1}{\pr(O)}\frac{\partial}{\partial p_x}\sum_{\theta\in\Theta}\pr(\theta)\cdot\frac{N(I_\theta\models O)}{N(\theta)}.
    \end{equation}
    We define the objective function as the log-likelihood subject to the restriction that all probabilities over $X$ sum to one, that is, let $\mathcal{X}$ be the set of all possible values taken by $X$, then the new objective function is the log-likelihood with the added Lagrange multiplier $\lambda$
    \begin{equation}\label{thm:lagrangian:lagrangian}
        \hat{\mathcal{L}}(\{p_x\}_{x\in\mathcal{X}},\lambda,\set{O})=\mathcal{L}(\set{O})-\lambda\left(\sum_{x\in\mathcal{X}}p_x-1\right).
    \end{equation}
    Optimization must now take place within the new constrained log-likelihood $\hat{\mathcal{L}}$.
    To find the complete expression of \eqref{thm:lagrangian:lagrangian}, we must find the value of the Lagrange multiplier $\lambda$, which is achievable by noting that the sum of all derivatives with respect to $X$ must sum to zero
    \begin{align}
    \begin{split}\label{thm:lagrangian:gradient}
        &\sum_{x\in\mathcal{X}}\frac{\partial}{\partial p_x}\hat{\mathcal{L}}(\{p_x\}_{x\in\mathcal{X}},\lambda,\set{O})=\sum_{x\in\mathcal{X}}\left(\frac{\partial}{\partial p_x}\mathcal{L}(\set{O})-\lambda\right)=\\
        &\sum_{x\in\mathcal{X}}\frac{\partial}{\partial p_x}\mathcal{L}(\set{O})-|\mathcal{X}|\cdot\lambda\implies\lambda=\frac{1}{|\mathcal{X}|}\sum_{x\in\mathcal{X}}\frac{\partial}{\partial p_x}\mathcal{L}(\set{O}).
    \end{split}
    \end{align}
    Let us first revisit \eqref{thm:lagrangian:lldx}, as it appears as a term of the gradient of our new objective function.
    We may split \eqref{thm:lagrangian:lldx} into two terms, one where total choices $\Theta_x$ agree with the assignment $X=x$, and the other where total choices $\Theta_{\overline{x}}$ choose other values $X=\overline{x}$, $\overline{x}\neq x$
    \begin{align}\label{thm:lagrangian:split}
    \begin{split}
        \text{\eqref{thm:lagrangian:lldx}}&=\\
        &\sum_{O\in\set{O}}\frac{1}{\pr(O)}\left(\frac{\partial}{\partial p_x}\sum_{\theta_x\in\Theta_x}\pr(\theta_x)\frac{N(I_{\theta_x}\models O)}{N(\theta_x)}+\right.\\
        &\left.\frac{\partial}{\partial p_x}\sum_{\theta_{\overline{x}}\in\Theta_{\overline{x}}}\pr(\theta_{\overline{x}})\frac{N(I_{\theta_{\overline{x}}}\models O)}{N(\theta_{\overline{x}})}\right).
    \end{split}
    \end{align} 
    We now derive each coordinate of the gradient of \eqref{thm:lagrangian:lagrangian}
    \begin{align}
    \begin{split}
        &\frac{\partial}{\partial p_x}\hat{\mathcal{L}}(\{p_x\}_{x\in\mathcal{X}},\lambda,\set{O})=\frac{\partial}{\partial p_x}\mathcal{L}(\set{O})-\lambda=\\
        &\eqref{thm:lagrangian:split}-\frac{1}{|\mathcal{X}|}\sum_{x'\in\mathcal{X}}\frac{\partial}{\partial p_{x'}}\mathcal{L}(\set{O}).
    \end{split}
    \end{align}
    To further simplify the above expression, we must further develop the formula in \eqref{thm:lagrangian:split}. Within the constraints set by the Lagrange multiplier, we may then simplify \eqref{thm:lagrangian:split} as
    \begin{equation}
        \eqref{thm:lagrangian:split}=\sum_{O\in\set{O}}\frac{1}{\pr(O)}\sum_{\theta_x\in\Theta_x}\frac{\pr(\theta_x)}{\pr(X=x)}\cdot\frac{N(I_{\theta_x}\models O)}{N(\theta_x)},
    \end{equation}
    since the second term in \eqref{thm:lagrangian:split} cancels out to zero. Note that we do not need to write $p_{\overline{x}}$ as a function of $p_x$ (which would mean that the term would not be equal to zero), as the Lagrange multiplier is already taking into account this relationship.
    Finally,
    \begin{align}
    \begin{split}
        &\eqref{thm:lagrangian:gradient}=\sum_{O\in\set{O}}\frac{1}{\pr(O)}\sum_{\theta_x\in\Theta_x}\frac{\pr(\theta_x)}{\pr(X=x)}\cdot\frac{N(I_{\theta_x}\models O)}{N(\theta_x)}-\\
        &\frac{1}{|\mathcal{X}|}\sum_{x'\in\mathcal{X}}\sum_{O\in\set{O}}\frac{1}{\pr(O)}\sum_{\theta_{x'}\in\Theta_x}\frac{\pr(\theta_{x'})}{\pr(X=x')}\cdot\frac{N(I_{\theta_{x'}}\models O)}{N(\theta_{x'})}
    \end{split}
    \end{align}
    can be rearranged so that when $x'=x$, the first and second terms are subtracted, yielding the final form
    \begin{align}
    \begin{split}
        &\frac{\partial}{\partial p_x}\hat{\mathcal{L}}(O)=\\
        & \, \left(1-\frac{1}{|\mathcal{X}|}\right)\frac{1}{\pr(O)}\sum_{\theta_x\in\Theta_x}\frac{\pr(\theta_x)}{\pr(X=x)}\cdot\frac{N(I_{\theta_x}\models O)}{N(\theta_x)}\\
        & \; -\frac{1}{|\mathcal{X}|}\sum_{\substack{\overline{x},\,\overline{x}\neq x}}\frac{1}{\pr(O)}\sum_{\theta_{\overline{x}}\in\Theta_{\overline{x}}}\frac{\pr(\theta_{\overline{x}})}{\pr(X=\overline{x})}\cdot\frac{N(I_{\theta_{\overline{x}}}\models O)}{N(\theta_{\overline{x}})};
    \end{split}
    \end{align}
    call $m=|\mathcal{X}|$ and we get the expression in our claim.
\end{proof}

\section{{\sysname} MNIST Addition Program}\label{sec:mnist}

\begin{listing*}
\begin{pasp}
% Addition of MNIST digits.

#python
import torch
import torchvision

class Net(torch.nn.Module):
  def __init__(self):
    super().__init__()
    self.encoder = torch.nn.Sequential(
      torch.nn.Conv2d(1, 6, 5),
      torch.nn.MaxPool2d(2, 2),
      torch.nn.ReLU(True),
      torch.nn.Conv2d(6, 16, 5),
      torch.nn.MaxPool2d(2, 2),
      torch.nn.ReLU(True)
    )
    self.classifier = torch.nn.Sequential(
      torch.nn.Linear(16 * 4 * 4, 120),
      torch.nn.ReLU(),
      torch.nn.Linear(120, 84),
      torch.nn.ReLU(),
      torch.nn.Linear(84, 10),
      torch.nn.Softmax(1)
    )

  def forward(self, x):
    x = self.encoder(x)
    x = x.view(-1, 16 * 4 * 4)
    x = self.classifier(x)
    return x

def digit_net(): return Net()

def mnist_data():
  train = torchvision.datasets.MNIST(root = "/tmp", train = True, download = True)
  test  = torchvision.datasets.MNIST(root = "/tmp", train = False, download = True)
  return train.data.float().reshape(len(train), 1, 28, 28)/255., train.targets, \
         test.data.float().reshape(len(test), 1, 28, 28)/255., test.targets

def normalize(X_R, Y_R, X_T, Y_T, mu, sigma):
  return (X_R-mu)/sigma, Y_R, (X_T-mu)/sigma, Y_T

train_X, train_Y, test_X, test_Y = normalize(*mnist_data(), 0.1307, 0.3081)
def pick_slice(data, which):
  h = len(data)//2
  return slice(h, len(data)) if which else slice(0, h)
def mnist_images_train(which): return train_X[pick_slice(train_X, which)]
def mnist_images_test(which): return test_X[pick_slice(test_X, which)]
def mnist_labels_train():
  labels = torch.concatenate((train_Y[:(h := len(train_Y)//2)].reshape(-1, 1),
                              train_Y[h:].reshape(-1, 1)), axis=1)
  return [[f"sum({x.item() + y.item()})"] for x, y in labels]
#end.

input(0) &$\sim$& test(@mnist_images_test(0)), train(@mnist_images_train(0)).
input(1) &$\sim$& test(@mnist_images_test(1)), train(@mnist_images_train(1)).

?::digit(X, {0..9}) as @digit_net with optim = "Adam", lr = 0.001 :- input(X).
sum(Z) :- digit(0, X), digit(1, Y), Z = X + Y.

#semantics maxent.
#learn @mnist_labels_train, lr = 1., niters = 5, alg = "lagrange", batch = 1000. 
\end{pasp}
\end{listing*}

\bibliography{refs}
\bibliographystyle{icml2023}

\end{document}